\documentclass[11pt]{article}

\usepackage[margin=1in]{geometry}
\usepackage[utf8]{inputenc}
\usepackage[T1]{fontenc}
\usepackage[numbers,sort&compress]{natbib}
\usepackage{hyperref}
\usepackage{url}
\usepackage{booktabs}
\usepackage{amsfonts}
\usepackage{nicefrac}
\usepackage{microtype}
\usepackage{xcolor}
\usepackage{amsmath,amssymb,bm}
\usepackage{graphicx}
\usepackage{multirow}
\usepackage{subcaption}

\usepackage{amsthm}

\newtheorem{theorem}{Theorem}[section]
\newtheorem{lemma}[theorem]{Lemma}
\newtheorem{corollary}[theorem]{Corollary}

\newtheorem{assumption}[theorem]{Assumption}

\theoremstyle{definition}
\newtheorem{definition}[theorem]{Definition}

\theoremstyle{remark}
\newtheorem{remark}[theorem]{Remark}

\title{Twist Flow for Inverse Problems}
\author{Shiqin Zeng \quad Zijun Deng \quad Felix J. Herrmann\\[0.5em]\small Georgia Institute of Technology}
\date{}

\begin{document}

\maketitle

\begin{abstract}


In Bayesian inverse problems, posterior sampling requires generating samples
that are consistent with given observations while capturing the range of
plausible solutions. Direct conditional generative models introduce latent
noise to model this ambiguity, but paired inverse-problem training can still
encourage an almost deterministic map from the observation to the target. As a
result, generated samples may be observation-consistent while
under-representing posterior variability, especially when the posterior is
multimodal, leading to undercoverage, mode distortion, or artificial
transitions between distinct feasible solutions. We propose \emph{joint twist-flow}, an augmented flow-matching formulation that
learns a continuous transport from the augmented source state
\((z_x,y)\) to the augmented terminal state \((x,z_y)\). Here \(x\) is the
target variable, \(y\) is the observation, \(z_x\) is the Gaussian reference
coordinate for posterior sampling, and \(z_y\) is a Gaussian likelihood-side
coordinate associated with the observation branch. Under a Gaussian observation
model, \(z_y\) is motivated by the normalized observation residual associated
with observation compatibility. Its role is not to replace uncertainty in
\(x\), but to couple generated samples of \(x\) to observation consistency,
helping reduce likelihood-inconsistent variation while preserving variability
in weakly constrained directions. We validate the method on low-dimensional inverse problems with reference
posterior samples, where joint twist-flow better preserves multimodal posterior
support than a direct conditional-flow baseline. We further evaluate the method
on image restoration and seismic subsurface velocity-model inversion, showing
increased posterior variability while maintaining observation consistency.
\end{abstract}
\section{Introduction}
\label{sec:intro}

Generative models have become a central tool for learning and sampling from
high-dimensional distributions, with major developments including generative
adversarial networks (GANs) \cite{goodfellow2014generative}, variational
autoencoders (VAEs) \cite{kingma2013auto}, normalizing flows
\cite{dinh2016density,kingma2018glow}, and diffusion and score-based generative
models \cite{ho2020denoising,song2020score}. In inverse problems, however,
the goal is not to sample from the marginal data distribution alone. Given an
observation $y$, one seeks samples of an unknown target $x$ from the posterior
distribution $p(x\mid y)$ \cite{stuart2010inverse}. This task is difficult
because the observation operator is often many-to-one, so a fixed observation
can be compatible with multiple plausible targets. Observation consistency is
necessary but not sufficient. A useful posterior sampler should also represent
the uncertainty that remains unresolved by the observation.

Conditional generative models address this ambiguity by introducing latent
variables. A common parameterization learns a direct conditional map from
Gaussian latent noise and the observation to the target. In principle, this
allows the model to generate different feasible solutions for the same
observation. However, in paired inverse-problem training, each observation
$y_i$ is often matched with only one target $x_i$. The training objective may
therefore be satisfied by relying mostly on the observation and using the latent
noise only weakly. As a result, the sampler may remain consistent with the
observation but produce too little posterior variation. When several distinct
solutions are compatible with the same observation, the sampler may miss some
solutions, distort others, or create unrealistic transitions between them.

We introduce \emph{joint twist-flow} as shown in Figure~\ref{fig:framework}, an augmented flow-matching formulation for
posterior sampling in inverse problems. Rather than learning a direct map from
latent noise and the observation to $x$, \emph{joint twist-flow} learns a
continuous augmented transport
\[
    S_\theta(z_x,y) = (x,z_y),
\]
where $z_x\sim\mathcal N(0,I)$ drives posterior sampling and $z_y$ is an
auxiliary terminal coordinate associated with the observation. This added
coordinate gives the transport an observation-side degree of freedom, reducing
the tendency to fit paired training data through a nearly deterministic
dependence of $x$ on $y$. For a fixed observation $y$, the learned flow maps
independent draws of $z_x\sim\mathcal N(0,I)$ to an ensemble of target samples
$x$. The variability across this ensemble represents the posterior uncertainty
left unresolved by the observation.


\begin{figure*}[t]
    \centering
    \includegraphics[width=\textwidth]{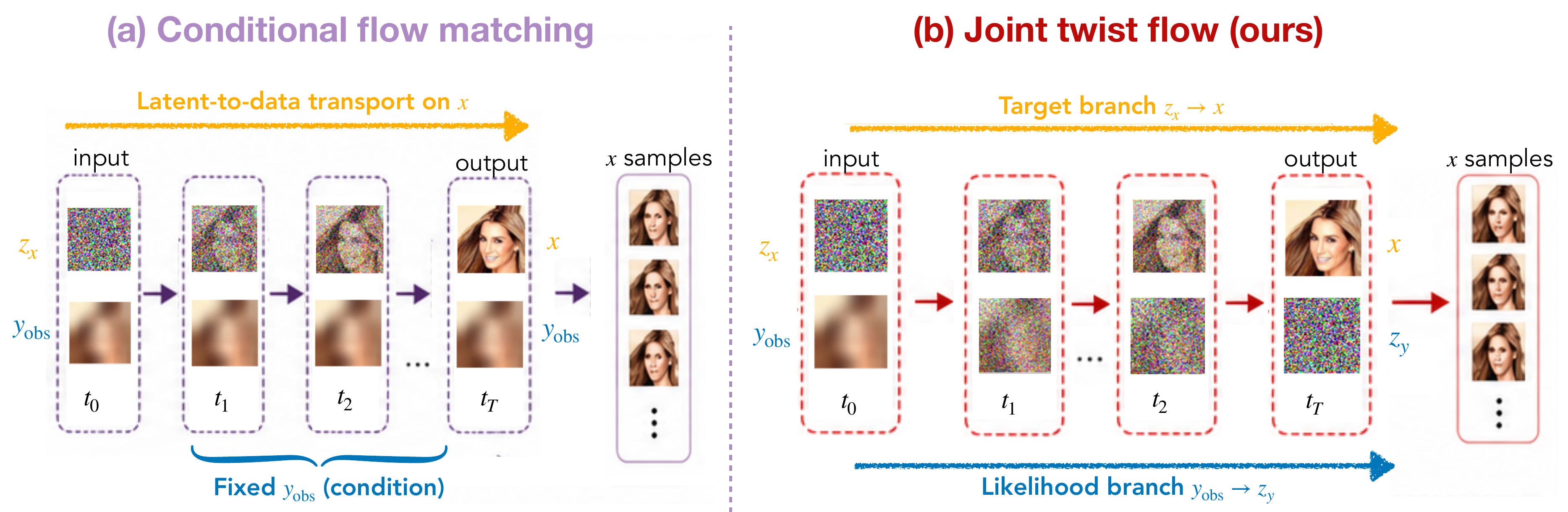}
    \caption{
    Overview of conditional flow matching and joint
    twist-flow.
    (a) Conditional flow matching samples $x$ from a one-way conditional map
    with fixed observation $y_{\mathrm{obs}}$, which may underrepresent
    posterior ambiguity.
    (b) Joint twist-flow learns an augmented transport $(z_x,y_{\mathrm{obs}})\mapsto(x,z_y)$, where $x$ is the posterior sample
    and $z_y$ is a Gaussian likelihood-side coordinate for the observation
    branch. This augmented transport is designed to preserve broader posterior
    variability under observation consistency.
    }
    \label{fig:framework}
\end{figure*}

This paper makes three contributions. First, we propose \emph{joint twist-flow},
an augmented parameterization $(z_x,y)\mapsto(x,z_y)$ for fixed-observation
posterior sampling. Second, we develop a flow-matching objective on the
augmented state and interpret $z_y$ as a likelihood-side coordinate associated
with observation consistency. Third, we evaluate whether \emph{joint twist-flow}
preserves posterior variation while maintaining observation consistency.
Low-dimensional experiments with reference posterior samples validate our
method by showing that \emph{joint twist-flow} reduces latent underuse and
better preserves posterior support. We then evaluate the same formulation on
high-dimensional image restoration and seismic velocity-model inversion,
assessing whether generated samples remain consistent with the fixed observation
while capturing a range of plausible solutions.
\section{Related work}
\label{sec:related_work}
\paragraph{Learned-prior methods for inverse problems}Generative methods for inverse problems combine learned priors or learned
generative dynamics with observation consistency. Plug-and-play (PnP) methods use learned
denoisers as implicit priors within iterative reconstruction algorithms
\cite{venkatakrishnan2013plug}, and regularization-by-denoising (RED)
uses denoisers to define explicit regularizers for inverse problems
\cite{romano2017little}. Diffusion and score-based inverse methods use pretrained
generative dynamics to guide reconstruction from observed data. Denoising
Diffusion Restoration Models (DDRM) restore images from degraded measurements
using denoising diffusion priors \cite{kawar2022denoising}. Diffusion Posterior
Sampling (DPS) combines diffusion sampling with observation consistency for
posterior-guided reconstruction \cite{chung2022diffusion}, while score-based inverse
solvers couple learned score dynamics with the forward observation operator at
inference time \cite{song2021solving}. These methods demonstrate the
strength of learned priors for observation-consistent reconstruction, often by
combining a learned prior with data-consistency or observation-guidance updates
at inference time. In contrast, \emph{joint twist-flow} learns an amortized posterior sampler as an
augmented flow-based transport for fixed-observation inference.
\paragraph{Flow-based posterior sampler parameterizations} Normalizing flows parameterize invertible maps between simple reference
distributions and complex data distributions, enabling reversible sampling and
likelihood-based training through change of variables
\cite{dinh2016density,kingma2018glow}. Inverse problems have also been addressed
with invertible neural networks, which use reversible architectures to represent
ambiguous inverse mappings and posterior uncertainty \cite{ardizzone2018analyzing}.
Continuous normalizing flows extend normalizing flows to continuous time by
representing the transformation from latent variables to data variables as the
solution of an ODE \cite{chen2018neural,grathwohl2018ffjord}. Flow matching
and rectified flow provide velocity-field objectives for such continuous-time
generative dynamics by learning vector fields that transport samples between
reference and data distributions along prescribed probability paths
\cite{lipman2022flow,liu2022flow}. Stochastic interpolants provide
a broader framework connecting flow and diffusion constructions through bridge
processes and associated velocity fields \cite{albergo2025stochastic}. Our
method builds on this continuous-time transport viewpoint, but applies it to the
augmented inverse-problem state $(z_x,y)\mapsto(x,z_y)$ rather than to a direct
conditional map from latent noise to the unknown variable alone.

\paragraph{Randomized least squares and transport for Bayesian inverse problems}
Classical Bayesian inverse problems often involve a Gaussian likelihood, in
which the normalized observation residual has standard Gaussian scale.
Optimization-based sampling methods exploit this structure. Weak-constraint
formulations introduce auxiliary variables to relax hard reduced formulations
and make posterior representations easier to characterize numerically
\cite{fang2018uncertainty}. Randomize-then-optimize (RTO) methods connect posterior
sampling with maps from Gaussian reference variables to optimization-based
posterior proposals \cite{bardsley2020scalable}. Prior-transformation methods
emphasize the usefulness of standardized Gaussian coordinates before applying
transport-based sampling schemes \cite{wang2017bayesian}. \emph{Joint twist-flow} is
closest in spirit to this residual-coordinate view: $z_x$ acts as a reference
coordinate for the unknown branch, while $z_y$ acts as a likelihood-side
residual coordinate paired with the observation branch. Unlike RTO-style
methods, however, \emph{joint twist-flow} learns this relation by flow matching instead
of solving a new randomized optimization problem for every sample.

\paragraph{Joint conditional generation}
Recent work has emphasized that posterior and likelihood conditionals can be
viewed as different conditionals of a shared joint distribution over unknown
variables and observations. All-in-one simulation-based inference methods learn
models capable of sampling multiple conditionals of this joint distribution
\cite{gloeckler2024all}, and related invertible generative models support
forward and inverse problems within one shared transport
\cite{van2025invertible}. Our work adopts this
shared-joint view of conditional generation, but realizes it through a
continuous-time flow-matching transport on the augmented state
$(z_x,y)\mapsto(x,z_y)$. Here $z_x$ drives posterior sampling under the fixed
observation, while $z_y$ is the terminal coordinate paired with the observation
branch. Rather than serving as additional observed information or replacing
uncertainty in $x$, $z_y$ provides a likelihood-side residual coordinate that
makes the transport full-dimensional and keeps generated targets coupled to
their conditioning observations.

\section{Method}
\label{sec:method}

\newcommand{\E}{\mathbb{E}}
\newcommand{\R}{\mathbb{R}}

\begin{figure}[t]
    \centering
    \includegraphics[width=\columnwidth]{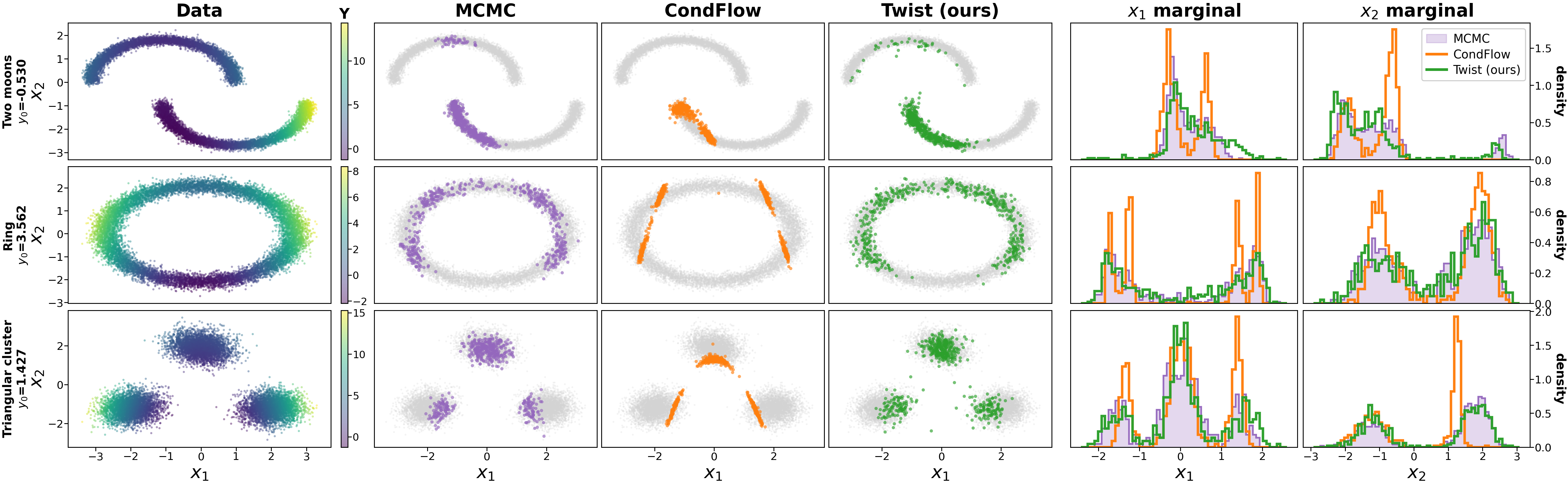}
    \caption{
    Toy posterior comparison for three inverse problems. Each row corresponds to
    one structured prior and one fixed observation $y_0$. The first column shows
    samples from the prior in $x$-space, with color indicating the scalar
    observation value $y$ generated by the forward model. The next three columns
    compare the MCMC posterior reference, conditional-flow samples, and joint
    twist-flow samples for the fixed $y_0$; gray points indicate the prior
    support for context. The last two columns compare the one-dimensional
    posterior marginals of $x_1$ and $x_2$.
    }
    \label{fig:toy_example}
\end{figure}
\vspace{-0.1cm}

\subsection{Problem setup}
We consider paired data $(x,y)\sim p_{\mathrm{data}}(x,y)$, where
$x\in\R^{d_x}$ is the unknown target variable and $y\in\R^{d_y}$ is the
observation. The pairs are generated by a physical or prescribed forward
relation
\begin{equation}
    y=\mathcal A(x)+\varepsilon,
\label{eq:forward-data}
\end{equation}
with the noiseless case corresponding to $\varepsilon=0$. The paired data
provide samples of the forward target--observation relation, but the inverse
map is generally ambiguous: a fixed observation $y$ may correspond to multiple
plausible targets $x$.

Posterior sampling therefore requires a conditional generator that can preserve
variability in $x$ while remaining compatible with the fixed observation. Under
a Gaussian observation model, this compatibility is measured by the whitened
residual
\begin{equation}
    r_y(x,y)
    =
    \Gamma_y^{-1/2}\bigl(y-\mathcal A(x)\bigr),
\label{eq:whitened-residual}
\end{equation}
which motivates representing the observation-side relation on a Gaussian scale.

\subsection{Joint twist-flow}

The setup above creates a parameterization challenge. A direct conditional flow
samples from an $x$-valued map
\begin{equation}
    x=G_\theta(z_x;y),
    \qquad
    z_x\sim\mathcal N(0,I_{d_x}),
\label{eq:direct-flow}
\end{equation}
where $y$ is used only as a conditioning input. Since the terminal state
contains only $x$, the same output branch must represent both posterior
variability and observation compatibility. In paired inverse-problem training,
this one-sided parameterization can favor a narrow reconstruction and
underrepresent the conditional variability of $x$.

Joint twist-flow addresses this by making the observation part of the
transported state. We learn an augmented transport
\begin{equation}
    S_\theta:(z_x,y)\mapsto(x,z_y),
    \qquad
    z_x\sim\mathcal N(0,I_{d_x}),
    \quad
    z_y\sim\mathcal N(0,\sigma_y^2I_{d_y}).
\label{eq:twist-map}
\end{equation}
The source state contains a Gaussian sampling coordinate and the fixed
observation, while the terminal state contains the generated target and a
Gaussian likelihood-side coordinate.

We introduce $z_y$ as a prescribed Gaussian terminal coordinate for the
observation branch. Motivated by the residual view of observation compatibility,
$z_y$ is sampled independently from $\mathcal N(0,\sigma_y^2I_{d_y})$ during
training and paired with the target variable $x$ at the terminal state. This
gives the likelihood-side branch a controlled Gaussian scale, while the learned
transport couples the observation branch with posterior sampling. Thus, we define $    s_0=(z_x,y),
    s_1=(x,z_y),$ and set $\pi_t(\cdot\mid s_0,s_1)$ be a probability path connecting $s_0$ and
$s_1$, and let $u_t(s_t\mid s_0,s_1)$ denote the corresponding target velocity.
We train a joint velocity field $u_\theta(s_t,t)$ by the regression objective
\begin{equation}
\mathcal L_{\mathrm{twist}}(\theta)
=
\E_{\substack{
(x,y)\sim p_{\mathrm{data}}\\
z_x\sim\mathcal N(0,I_{d_x}),\;
z_y\sim\mathcal N(0,\sigma_y^2I_{d_y})\\
t\sim\mathcal U[0,1],\;
s_t\sim \pi_t(\cdot\mid s_0,s_1)
}}
\left[
\left\|
u_\theta(s_t,t)-u_t(s_t\mid s_0,s_1)
\right\|_2^2
\right].
\label{eq:twist-fm-loss}
\end{equation}

After training, the transport can be written as
\begin{equation}
    S_\theta(z_x,y)
    =
    \bigl(
    x_\theta(z_x;y),\,
    z_{y,\theta}(z_x;y)
    \bigr).
\label{eq:trained-transport}
\end{equation}
For a fixed observation $y$, $x_\theta(z_x;y)$ is the generated posterior
sample, while $z_{y,\theta}(z_x;y)$ is the associated learned likelihood-side
coordinate. Posterior sampling uses the $x$-marginal; $z_{y,\theta}$ is
generated jointly as part of the augmented transport.

The parameter $\sigma_y$ controls the prescribed Gaussian scale of the
likelihood-side terminal coordinate. For confidence level $q\in(0,1)$, define
\begin{equation}
    \mathcal B_q(\sigma_y)
    =
    \left\{
    z\in\R^{d_y}:
    \|z\|_2^2
    \leq
    \sigma_y^2\chi^2_{d_y}(q)
    \right\},
\label{eq:likelihood-side-ball}
\end{equation}
where $\chi^2_{d_y}(q)$ is the $q$-quantile of a chi-square distribution.
This ball gives an interpretable Gaussian scale for the likelihood-side
coordinate generated by the learned transport. It expresses the likelihood-aware
role of the augmented coordinate: generated samples may explore variability in
$x$, while their jointly generated likelihood-side coordinate remains organized
around the prescribed Gaussian scale.


For posterior sampling, the pretrained map \(S_\theta(\cdot,y)\) induces a joint terminal distribution
\(p_\theta(x,z_y\mid y)\). Equivalently, the posterior sampler is the
\(x\)-marginal of this augmented distribution,
\begin{equation}
    p_\theta(x\mid y)
    =
    \int p_\theta(x,z_y\mid y)\,dz_y .
\label{eq:x-marginal}
\end{equation}

Thus, the likelihood-side coordinate is amortized through the learned transport:
during training, the prescribed Gaussian variable \(z_y\) defines the terminal
coordinate paired with \(x\), while posterior sampling requires only \(z_x\)
and the fixed observation \(y\). The conditional sampler is learned from paired
target--observation data. When these pairs are sampled from a Bayesian joint
model \(p(x)p(y\mid x)\), the data conditional corresponds to the posterior
\(p(x\mid y)\). In this sense, joint twist-flow learns a posterior sampler
through the \(x\)-marginal of the augmented transport. The intended effect of
the likelihood-side coordinate is to improve the conditional transport
parameterization so that the \(x\)-marginal can retain posterior variability
while remaining tied to the learned target--observation relation; the
experiments test this through posterior-reference comparisons and
observation-space consistency diagnostics.
\section{Experiments}
\vspace{-0.1cm}
\label{sec:experiments}
\begin{table}[t]
\centering
\small
\caption{
Quantitative comparison against the MCMC posterior reference on the toy inverse
problems. The coefficients $(c_1,c_2)$ used in the nonlinear observation model
are listed below each setting.
}
\label{tab:toy_mcmc_quantitative}
\resizebox{\columnwidth}{!}{
\begin{tabular}{llcccc}
\toprule
Setting & Method
& $W_1(x_1)$
& $W_1(x_2)$
& KS$(x_1)$
& KS$(x_2)$ \\
\midrule
\begin{tabular}[c]{@{}l@{}}
Two moons\\[-0.5mm]
\scriptsize $c_1=1.0,\ c_2=0.5$
\end{tabular}
& Conditional Flow
& $0.282 \pm 0.246$
& $0.663 \pm 0.327$
& $0.269 \pm 0.127$
& $0.311 \pm 0.094$ \\
& Joint Twist Flow (ours)
& $\mathbf{0.132 \pm 0.030}$
& $\mathbf{0.151 \pm 0.072}$
& $\mathbf{0.121 \pm 0.014}$
& $\mathbf{0.074 \pm 0.016}$ \\
\midrule
\begin{tabular}[c]{@{}l@{}}
Ring\\[-0.5mm]
\scriptsize $c_1=1.2,\ c_2=0.8$
\end{tabular}
& Conditional Flow
& $0.181 \pm 0.019$
& $0.188 \pm 0.041$
& $0.133 \pm 0.022$
& $0.161 \pm 0.085$ \\
& Joint Twist Flow (ours)
& $\mathbf{0.098 \pm 0.019}$
& $\mathbf{0.133 \pm 0.012}$
& $\mathbf{0.056 \pm 0.007}$
& $\mathbf{0.101 \pm 0.025}$ \\
\midrule
\begin{tabular}[c]{@{}l@{}}
Triangular cluster\\[-0.5mm]
\scriptsize $c_1=1.5,\ c_2=1.1$
\end{tabular}
& Conditional Flow
& $0.220 \pm 0.104$
& $0.271 \pm 0.238$
& $0.147 \pm 0.051$
& $0.231 \pm 0.219$ \\
& Joint Twist Flow (ours)
& $\mathbf{0.142 \pm 0.020}$
& $\mathbf{0.129 \pm 0.072}$
& $\mathbf{0.091 \pm 0.025}$
& $\mathbf{0.059 \pm 0.015}$ \\
\bottomrule
\end{tabular}
}
\end{table}


\subsection{Toy inverse problems: posterior support preservation}
\label{sec:toy_validation}
\begin{figure*}[!htbp]
    \centering
    \begin{subfigure}[t]{\textwidth}
        \centering
        \includegraphics[width=0.87\linewidth]{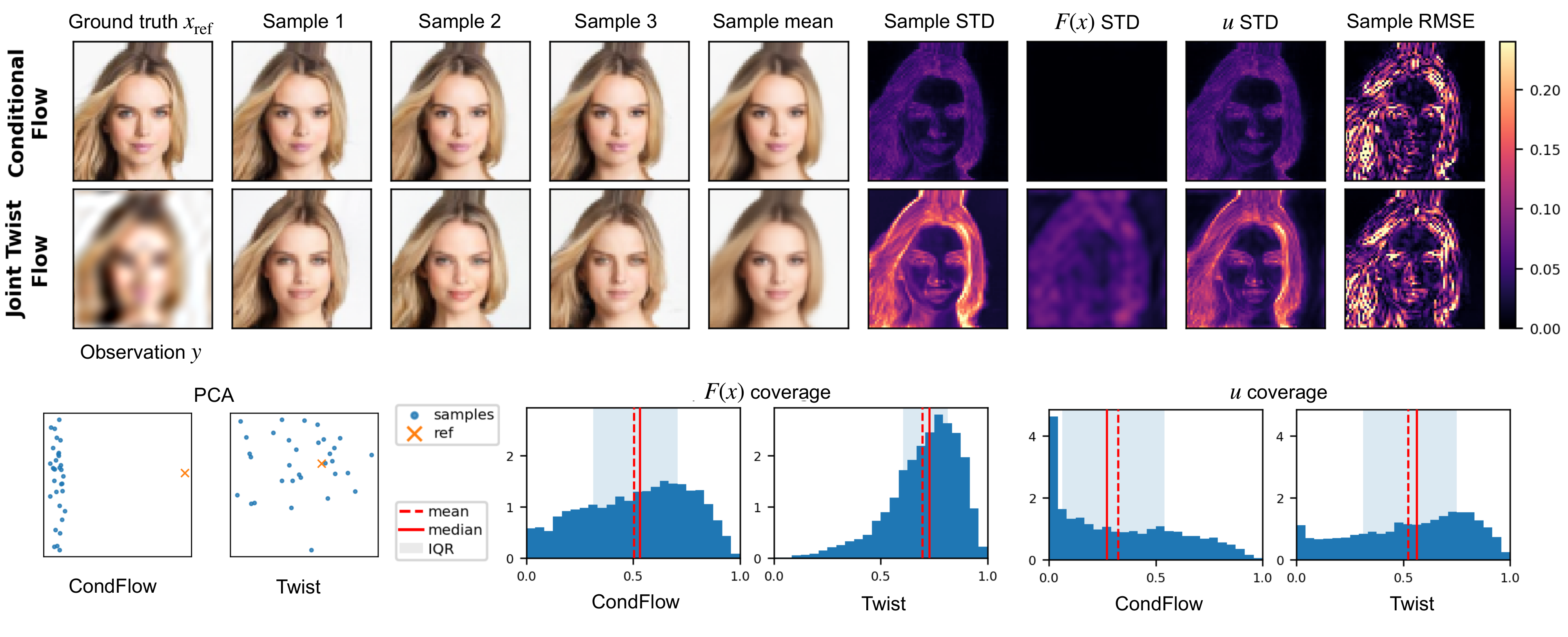}
        \caption{Low-frequency image restoration under a Fourier low-pass observation. Joint twist-flow produces a broader sample ensemble than CondFlow, with variability primarily in unresolved high-frequency components while the observed low-frequency structure remains stable.}
        \label{fig:lf}
    \end{subfigure}
    \hfill
    \begin{subfigure}[t]{\textwidth}
        \centering
        \includegraphics[width=0.87\linewidth]{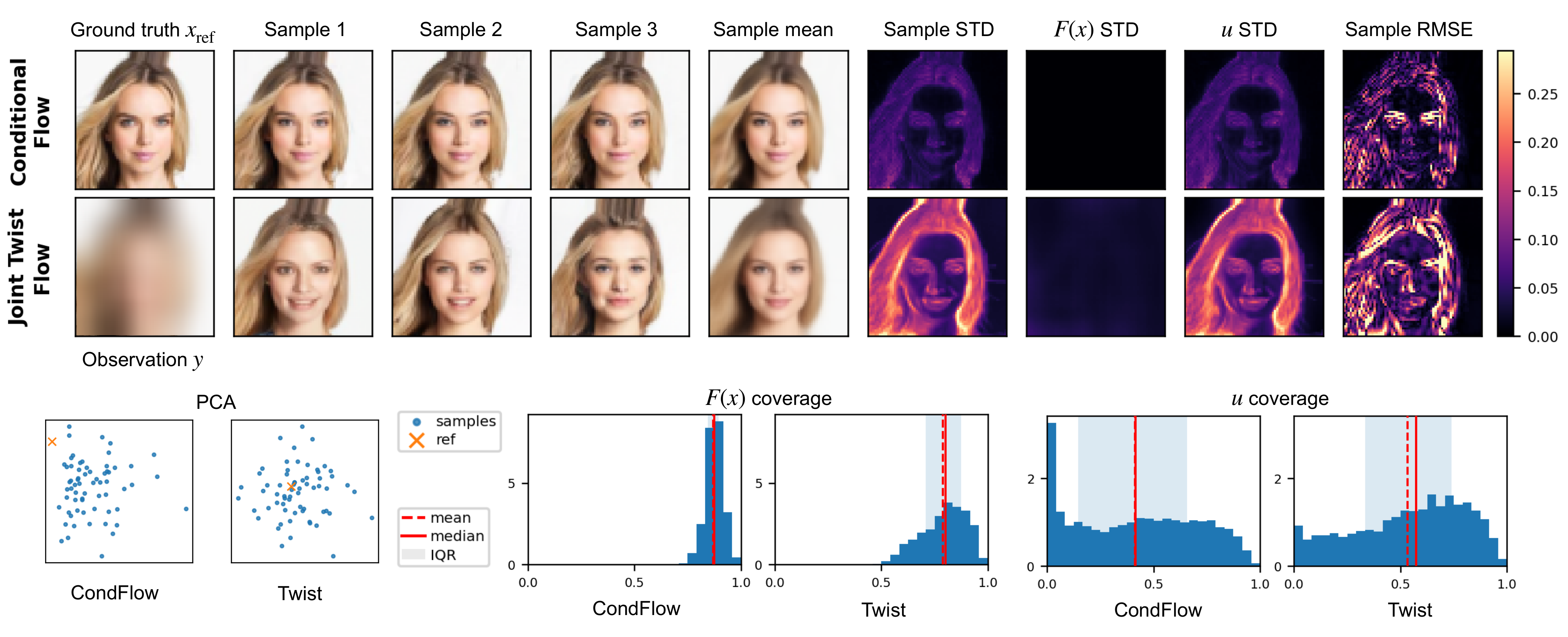}
        \caption{Blur-downsampling image restoration. Joint twist-flow generates more diverse high-resolution reconstructions than CondFlow, with uncertainty concentrated in fine-scale details not determined by the low-resolution observation.}
        \label{fig:downup}
    \end{subfigure}
    \hfill
    \begin{subfigure}[t]{\textwidth}
        \centering
        \includegraphics[width=0.87\linewidth]{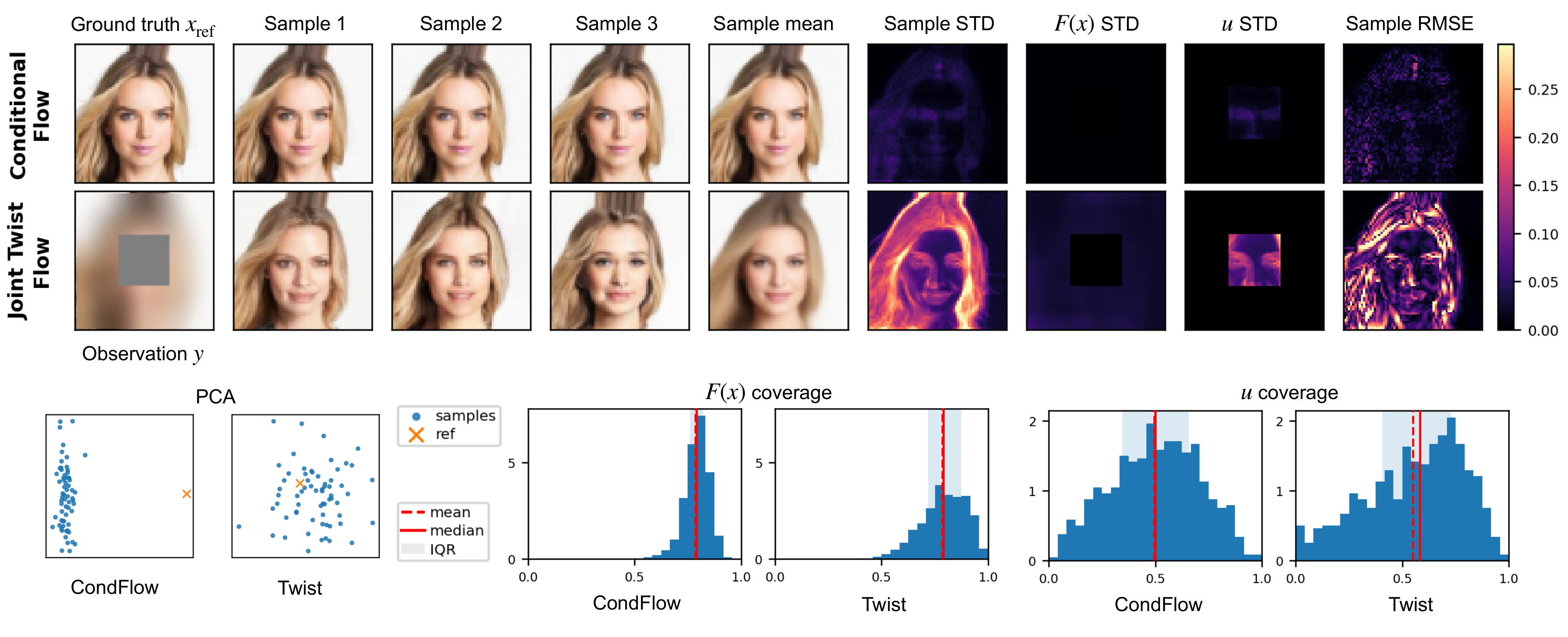}
        \caption{Blur-masking image restoration. Joint twist-flow concentrates posterior variability in the missing or weakly constrained region while keeping the observed region consistent with the fixed observation.}
        \label{fig:mask}
    \end{subfigure}
    \caption{Image restoration under three observation operators.}
    \label{fig:image_restoration}
\end{figure*}

We first evaluate joint twist-flow on controlled two-dimensional inverse
problems with reference posterior samples. The toy problems are designed to
expose a specific failure mode of direct conditional sampling: samples may be
associated with the correct observation while covering only part of the
posterior support or placing mass in low-density regions between feasible
solutions.

We use three structured priors over $x=(x_1,x_2)\in\mathbb{R}^2$:
\texttt{two\_moons}, \texttt{ring}, and \texttt{triangular cluster}. These priors create
curved, closed, and separated supports. For each prior, observations are
generated by 
\[
    y = c_1 x_1^2 + c_2 x_2 + \varepsilon,
    \qquad
    \varepsilon \sim \mathcal N(0,\sigma^2).
\]
The coefficients $(c_1,c_2)$ used for each setting are reported in
Table~\ref{tab:toy_mcmc_quantitative}. The first column of
Figure~\ref{fig:toy_example} visualizes this construction: point locations show
samples of $x$, while color indicates the corresponding scalar observation
value $y$. The scalar nonlinear observation is many-to-one,
so a fixed value $y_0$ can be compatible with multiple prior-supported regions.
From a Bayesian perspective, the likelihood selects the regions of the prior
that are compatible with $y_0$, producing multimodal or disconnected posterior
support. We compare joint twist-flow with a direct conditional-flow baseline. For each
fixed observation $y_0$, both models generate posterior samples, which are
compared against Markov chain Monte Carlo (MCMC) reference samples. Details of
data generation and MCMC sampling are provided in
Appendix~\ref{app:toy_data_generation} and Appendix~\ref{app:toy_mcmc}.

Table~\ref{tab:toy_mcmc_quantitative} and Figure~\ref{fig:toy_example} show
that joint twist-flow better matches the MCMC posterior reference. The table
quantifies marginal posterior mismatch: joint twist-flow reduces both
Wasserstein and Kolmogorov--Smirnov discrepancies for $x_1$ and $x_2$ in all
three settings. These improvements agree with the marginal plots in
Figure~\ref{fig:toy_example}, where the twist-flow marginals more closely
follow the MCMC reference than the conditional-flow marginals. The visual comparisons reveal the failure modes more directly. In
\texttt{two\_moons}, the conditional-flow baseline covers only a small portion
of the curved posterior support, leading to missing branches and mismatched
marginals. Joint twist-flow recovers a broader portion of the curved feasible
regions. In \texttt{ring}, the conditional-flow samples collapse to narrow arcs
instead of following the closed posterior support, which appears as sharp peaks
and gaps in the marginals. Joint twist-flow better preserves the ring-shaped
support and produces marginals closer to the MCMC reference. In
\texttt{triangular cluster}, the conditional-flow baseline produces elongated structures
near or between separated clusters, indicating off-support mass and distorted
modes. Joint twist-flow keeps samples concentrated near the three feasible
posterior regions and better matches the reference marginals.

Together, these results show that the toy problems expose the intended
posterior-support failure mode of direct conditional sampling. Joint twist-flow
reduces latent underuse and improves support preservation under paired training,
rather than only producing samples that are associated with the fixed
observation.


\vspace{-0.2cm}
\subsection{Image restoration: observation-aware posterior variability}
\vspace{-0.1cm}
\label{sec:image_restoration}
We next evaluate joint twist-flow on image restoration tasks. For each test
case, we have a clean reference image $x_{\rm ref}$ and generate the observation
as $y=F(x_{\rm ref})$, but we do not have reference samples from the true
posterior $p(x\mid y)$. Since the degradation operators are many-to-one,
$x_{\rm ref}$ is only one plausible solution compatible with $y$. We therefore
evaluate whether generated samples remain observation-consistent while
expressing variability in components not determined by the observation.

We use CelebA\cite{liu2015deep} images as clean targets $x$. Images are center-cropped, resized to
$64\times64$, and normalized to $[-1,1]$. Observations are generated by
operators $y=F(x)$: Fourier low-pass filtering, blur-downsampling, and
blur-masking. We use $u$ to denote the task-specific component weakly
constrained by the observation: high-frequency residual
$x-F_{\mathrm{LF}}(x)$ for low-frequency filtering, fine-scale detail
$x-F_{\mathrm{BD}}(x)$ for blur-downsampling, and missing-region content
$(1-M)x$ for blur-masking. Thus, the three operators preserve different parts
of the image and leave different components ambiguous.

Figures~\ref{fig:lf}, \ref{fig:downup}, and \ref{fig:mask} compare CondFlow and
joint twist-flow using the same layout. We show posterior samples, sample mean,
sample standard deviation, RMSE maps relative to the paired reference,
principal component analysis (PCA) projections, and coverage diagnostics for
both $F(x)$ and $u$. The sample standard deviation and PCA plots characterize
the overall spread of the generated ensemble. The $F(x)$ standard-deviation
maps, together with the coverage of the observed component, indicate whether
this spread remains consistent with the fixed observation. The corresponding
maps and coverage diagnostics for $u$ show whether variability appears in the
component weakly constrained by the observation.

Across all three tasks, CondFlow forms narrow PCA clusters and has low sample
standard deviation, indicating under-dispersed conditional samples. Its
$F(x)$ standard deviation is also small, but its $u$ standard deviation is small
as well, showing that little diversity is expressed in the component not
determined by the observation. In low-frequency filtering and
blur-downsampling, this narrow ensemble still shows visible RMSE relative to the
paired reference, consistent with a reconstruction that averages unresolved
high-frequency or fine-scale details. In blur-masking, CondFlow is similarly
under-dispersed, but the visible high-resolution context strongly constrains the
missing region, so the narrow reconstruction remains close to the paired
reference and yields lower RMSE.

Joint twist-flow shows a different pattern. It produces larger sample standard
deviation and broader PCA clouds, indicating a wider generated ensemble. Its
$F(x)$ standard-deviation maps remain small relative to the image-space spread,
and the $I_y$ values in Table~\ref{tab:image_restoration_summary} remain high,
showing that generated samples remain consistent with the fixed observation. At
the same time, the $u$ standard-deviation maps are larger and the $I_r$ values
improve, indicating that the increased variability appears in the component weakly
constrained by the observation. Because $x_{\rm ref}$ is only one plausible solution, larger sample-to-reference
RMSE should not be interpreted as worse posterior sampling by itself. Instead,
together with high $I_y$ and increased variability in $u$, it indicates that
twist-flow explores multiple observation-compatible solutions rather than
collapsing to the paired reference.
Quantitatively, Table~\ref{tab:image_restoration_summary} shows the same trend
over three test cases: joint twist-flow consistently increases sample spread
and $I_r$ while maintaining high $I_y$. In blur-downsampling, image-space
pairwise RMSE increases from $0.069$ to $0.202$, and $I_r$ improves from
$0.710$ to $0.881$; the low-frequency and blur-masking tasks follow the same
pattern. Together with the PCA, standard-deviation, and coverage diagnostics,
these results support the intended behavior of the augmented transport: joint
twist-flow increases posterior variability while organizing it around
compatibility with the fixed observation.

\begin{table*}[t]
\centering
\scriptsize
\setlength{\tabcolsep}{2.0pt}
\renewcommand{\arraystretch}{1.04}
\caption{
Observation-aware posterior variability on CelebA restoration tasks.
Each task uses observations of the form $y=F(x)$, while $u$ denotes the
task-specific weakly constrained component. Here $\mathcal F$ denotes the
Fourier transform, $P_{\Omega}$ is a centered low-frequency Fourier mask,
$B$ is the blur operator, $D$ and $U$ denote downsampling and upsampling,
$M$ is the observed-region mask, and $c$ is the constant fill value in the
masked region. Metrics are averaged over three test cases. $D_x$ is
image-space pairwise RMSE across samples, $S_x$ is mean sample standard
deviation, and $D_r$ is pairwise RMSE in $u$. $I_y$ and $I_r$ are empirical
90\% interval inclusion rates in the observed component and in $u$,
respectively.
}
\label{tab:image_restoration_summary}
\begin{tabular}{
@{}
l
p{0.315\textwidth}
p{0.145\textwidth}
l
ccccc
@{}
}
\toprule
Task
& Observation $y=F(x)$
& Weak component $u$
& Model
& $D_x \uparrow$
& $S_x \uparrow$
& $D_r \uparrow$
& $I_y \uparrow$
& $I_r \uparrow$ \\
\midrule

Low-frequency
&
$y=F_{\mathrm{LF}}(x)=\mathcal{F}^{-1}(P_{\Omega}\mathcal{F}x)$
&
$u=x-F_{\mathrm{LF}}(x)$
& CondFlow
& 0.062
& 0.036
& 0.062
& 0.821
& 0.594 \\
&
&
&
Twist
& 0.156
& 0.093
& 0.126
& 0.977
& 0.848 \\

\midrule

Blur-downsampling
&
$y=F_{\mathrm{BD}}(x)=U(D(B(x)))$
&
$ u=x-F_{\mathrm{BD}}(x)$
& CondFlow
& 0.069
& 0.041
& 0.069
& 1.000
& 0.710 \\
&
&
&
Twist
& 0.202
& 0.121
& 0.193
& 0.996
& 0.881 \\

\midrule

Blur-masking
&
$y=F_{\mathrm{M}}(x)=MB(x)+(1-M)c$
&
$u=(1-M)x$
& CondFlow
& 0.036
& 0.022
& 0.036
& 1.000
& 0.845 \\
&
&
&
Twist
& 0.212
& 0.127
& 0.165
& 0.998
& 0.889 \\

\bottomrule
\end{tabular}
\end{table*}
\vspace{-0.2cm}
\subsection{Seismic inverse problem}
\label{sec:seismic}
\vspace{-0.1cm}

We further evaluate joint twist-flow on an image-domain seismic inverse problem.
The target variable $x$ is a two-dimensional subsurface velocity model sampled
on a $256\times512$ grid with $12.5\,\mathrm{m}$ spacing, corresponding to a
section of approximately $3.2\,\mathrm{km}$ in depth and $6.4\,\mathrm{km}$ in
the horizontal direction. In a full seismic experiment, observed shot records
$d_{\mathrm{obs}}$ are generated by wave propagation through the subsurface, $  d_{\mathrm{obs}} = \mathcal{F}(x) + \epsilon ,$ where $\mathcal{F}$ denotes the acoustic wave-equation modeling operator and
$\epsilon$ represents noise and modeling error. Rather than conditioning the
generative model directly on the raw shot records, we use a reverse-time
migration (RTM) image as the observation $y$. RTM maps recorded wavefields back to the subsurface using a fixed background velocity model and produces an image-domain summary of reflector information \cite{baysal1983reverse, etgen2009overview}. This setting remains strongly ill posed from a geophysical perspective. The RTM image is band-limited, acquisition-dependent, and sensitive to the background
model used during imaging \cite{zhou2018reverse}. It highlights reflector structure but does not
uniquely determine the full velocity model, especially in deeper regions,
thin-layered zones, and areas where illumination is limited. The experiment
therefore tests whether the model can produce an RTM-conditioned ensemble of
velocity models rather than a single overconfident reconstruction. Seismic data
are simulated using \textit{Devito}~\cite{louboutin2019devito} and
\textit{JUDI.jl}~\cite{witte2019large}, with simulation details provided in
Appendix~\ref{app:seis_details}.

Figure~\ref{fig:seismic} compares the conditional-flow baseline and joint
twist-flow under the same fixed RTM observation. Both methods recover the
dominant layered structure and the deep high-velocity region in the posterior
mean. The main difference lies in the posterior spread. CondFlow produces a
narrow ensemble with low standard deviation over most of the model, indicating
an overconfident reconstruction. Joint twist-flow produces larger and more
structured uncertainty, concentrated around reflector interfaces, thin layers,
structural transitions, and deeper regions where the RTM image provides weaker
constraints. The RMSE maps show larger sample-to-reference deviations for
twist-flow, but these deviations align with the same uncertain geological
regions rather than appearing as unstructured noise. The bottom row of Figure~\ref{fig:seismic} provides an observation-space check.
Generated velocity samples are forward modeled under the same acquisition
geometry, and the resulting shot-record bands are compared with the
ground-truth shot record. Joint twist-flow yields a wider shot-record band that
better covers the ground-truth trace in this example, while the receiver-wise
coverage profile remains high across the acquisition aperture. This indicates
that the larger model-space variability does not simply discard the seismic
observation; it produces a broader set of velocity models that remain compatible
with the observed data.

Table~\ref{tab:seismic_metrics} summarizes reference metrics over 64 posterior
samples. CondFlow gives lower per-sample error and higher per-sample SSIM,
consistent with a narrow reconstruction-like ensemble clustered near the paired
reference. Joint twist-flow has larger per-sample error and lower per-sample
SSIM, reflecting a broader posterior ensemble rather than collapse to the single
paired reference. Importantly, the posterior-mean RMSE and SSIM remain close to
those of CondFlow, indicating that the additional variability mainly increases
ensemble spread rather than shifting the mean reconstruction. Together with the
shot-record consistency check, this suggests that joint twist-flow produces a
broader observation-compatible posterior ensemble around a comparable velocity
reconstruction.
\begin{figure}[t]
    \centering
    \includegraphics[width=1.03\columnwidth]{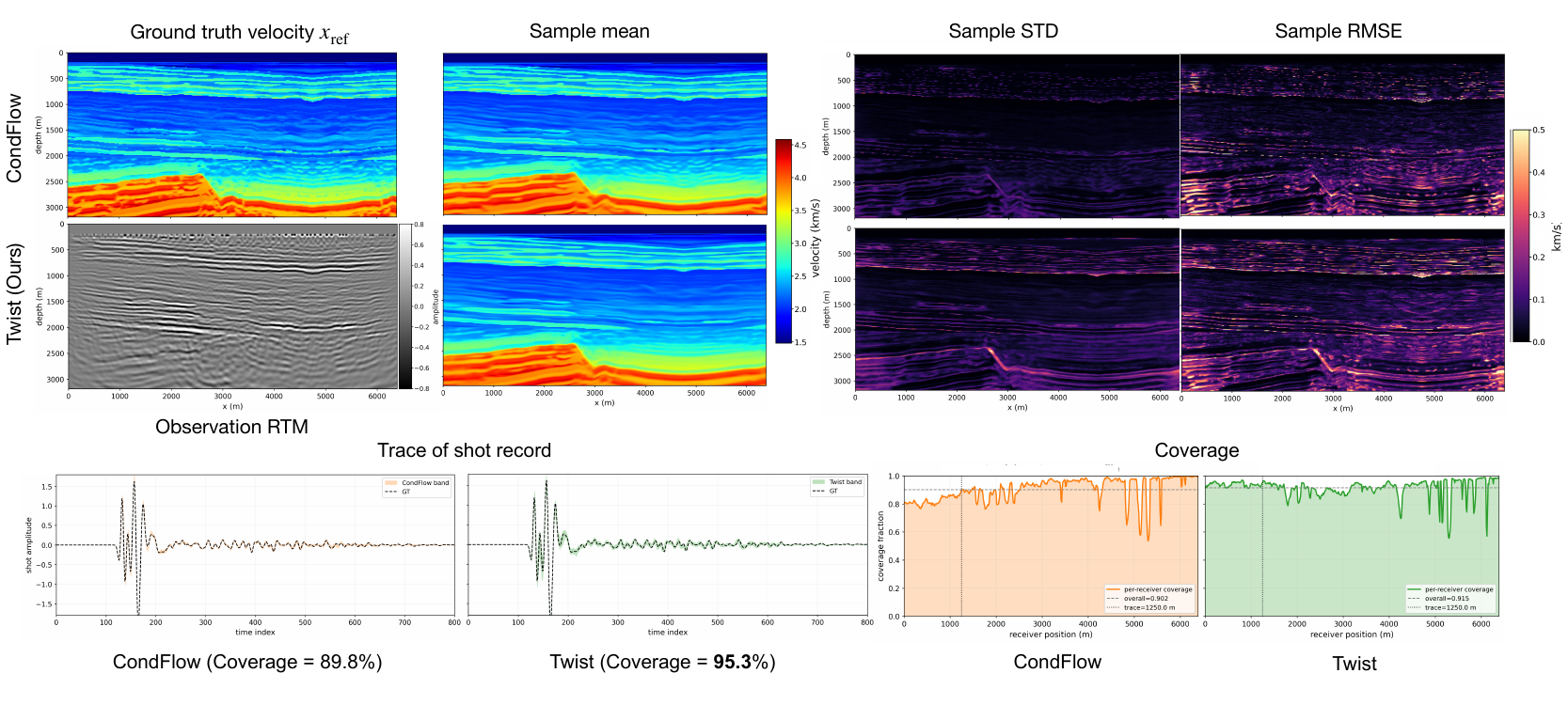}
    \caption{Seismic posterior sampling under a fixed RTM observation. The top rows show
    the ground-truth velocity and RTM image, posterior means, standard-deviation
    maps, and RMSE maps for CondFlow and joint twist-flow. Joint twist-flow preserves the main velocity structure while assigning larger uncertainty to reflector interfaces, structural transitions, and deeper weakly constrained regions. The bottom row forward models generated velocity samples under the
    same acquisition; shot-record bands and receiver-wise coverage profiles provide an observation-space consistency check.}
    \label{fig:seismic}
\end{figure}
\vspace{-0.2cm}

\begin{table}[!htbp]
\centering
\footnotesize
\setlength{\tabcolsep}{3pt}
\renewcommand{\arraystretch}{1.02}
\caption{Seismic reference metrics over 64 posterior samples}
\label{tab:seismic_metrics}
\begin{tabular*}{0.98\columnwidth}{@{\extracolsep{\fill}}lcccc@{}}
\toprule
& \multicolumn{2}{c}{Per-sample} 
& \multicolumn{2}{c}{Posterior mean} \\
\cmidrule(lr){2-3}\cmidrule(lr){4-5}
Metric & CondFlow & Twist & CondFlow & Twist \\
\midrule
RMSE & $0.101 \pm 0.001$ & $0.118 \pm 0.005$ & $0.0919$ & $0.0921$ \\
MAE  & $0.0565 \pm 0.0006$ & $0.0680 \pm 0.0027$ & $0.0512$ & $0.0564$ \\
SSIM & $0.834 \pm 0.003$ & $0.777 \pm 0.011$ & $0.857$ & $0.841$ \\
\bottomrule
\end{tabular*}
\end{table}
\vspace{-0.2cm}
\section{Conclusions}
\label{sec:conclusions}
\vspace{-0.1cm}
We proposed joint twist-flow, an augmented flow-matching framework for posterior
sampling in inverse problems. The central idea is to learn a continuous
transport on the augmented state $(z_x,y)\mapsto(x,z_y)$ rather than an
$x$-valued conditional decoder. In this formulation, $z_x$ drives posterior
sampling under a fixed observation, while $z_y$ provides a likelihood-side
residual coordinate paired with the observation branch. The auxiliary coordinate
does not replace uncertainty in $x$; posterior uncertainty remains measured by
the variability of generated $x$-samples. The experiments support three related conclusions. First, in controlled toy
inverse problems, joint twist-flow reduces latent underuse, mode dropping, and
support distortion relative to a direct conditional-flow baseline. Second, in
image restoration, the augmented transport produces observation-aware posterior
variability: uncertainty increases in weakly constrained components while
remaining suppressed in components that determine the observation. Third, the
seismic experiment shows that the same formulation can be applied to a
high-dimensional scientific inverse problem where a single observation may be
compatible with multiple plausible subsurface structures. A promising direction for future work is to improve calibration diagnostics for
amortized posterior samplers, and explore
adaptive choices of the observation-side scale $\sigma_y$. Another direction is
to use the reverse direction of the learned augmented transport more explicitly
for likelihood-side diagnostics and forward uncertainty propagation.

{\small
\bibliographystyle{unsrtnat}
\bibliography{references}
}

\appendix
\clearpage

\section{Supplementary materials}
\label{sec:rationale}

\subsection{Topological Obstruction in Conditional Inverse Transport}
\label{app:topological_obstruction}

In this appendix, we formalize a topological obstruction for direct conditional
transport in inverse problems, and explain why the proposed joint twist transport can reduce
the parameterization pressure that gives rise to this obstruction. Our goal is not
to claim that the augmented map removes the topological difficulty in the
projected $x$-space, but to motivate why coupling the target branch with a
likelihood-side coordinate can make the learned conditional transport less prone
to empirical failure modes such as mode dropping, spurious bridges, or
off-support mass.

\subsubsection{Setup}

We consider the inverse observation model
\begin{equation}
x \sim p_X(x), 
\qquad 
y = h(x) + \varepsilon,
\label{eq:inverse_model}
\end{equation}
where $x \in \mathcal X \subset \mathbb R^{d_x}$, $y \in \mathcal Y \subset \mathbb R^{d_y}$, $h:\mathcal X \to \mathcal Y$ is a continuous forward operator, and $\varepsilon$ is observational noise with density $p_\varepsilon$.

The induced joint density is
\begin{equation}
p(x,y) = p_X(x)\, p_\varepsilon\!\left(y-h(x)\right),
\label{eq:joint_density}
\end{equation}
and the posterior satisfies
\begin{equation}
p(x\mid y) \propto p_X(x)\, p_\varepsilon\!\left(y-h(x)\right).
\label{eq:posterior_general}
\end{equation}
For isotropic Gaussian noise, $\varepsilon \sim \mathcal N(0,\sigma^2 I)$, this becomes
\begin{equation}
p(x\mid y)
\propto
p_X(x)\exp\!\left(-\frac{\|y-h(x)\|^2}{2\sigma^2}\right).
\label{eq:posterior_gaussian}
\end{equation}

We compare the following two parameterizations.

\paragraph{Direct conditional transport.}
$$
G : \mathbb R^{d_z} \times \mathbb R^{d_y} \to \mathbb R^{d_x},
\qquad
x = G(z,y).
$$

\paragraph{joint twist transport.}
$$
T : \mathbb R^{d_z} \times \mathbb R^{d_y}
\to
\mathbb R^{d_x} \times \mathbb R^{d_{z_y}},
\qquad
(x,z_y) = T(z_x,y).
$$

Our main interest is the geometry of the posterior support for a fixed observation $y=y_0$.

\begin{definition}[Posterior superlevel set]
For a fixed $y_0 \in \mathcal Y$ and threshold $\lambda > 0$, define the posterior $\lambda$-superlevel set by
$$
L_{y_0,\lambda}
:=
\left\{
x \in \mathcal X : p(x\mid y_0)\ge \lambda
\right\}.
$$
\end{definition}

\begin{definition}[Latent superlevel set and its image]
Let $p_Z$ denote the latent density, and for $\lambda_Z>0$ define
$$
L^Z_{\lambda_Z}
:=
\left\{
z \in \mathbb R^{d_z} : p_Z(z)\ge \lambda_Z
\right\}.
$$
For a fixed $y_0$, define the fiber map
$$
G_{y_0}(z) := G(z,y_0),
$$
and the image of the latent high-density region
$$
\mathcal G_{y_0,\lambda_Z}
:=
G_{y_0}(L^Z_{\lambda_Z}).
$$
\end{definition}

\begin{definition}[Spurious bridge]
A connected set $B \subset \mathcal G_{y_0,\lambda_Z}$ is called a \emph{spurious bridge} if
\begin{enumerate}
    \item $B$ intersects neighborhoods of two distinct posterior branches, and
    \item $B$ is not contained in the true posterior superlevel set $L_{y_0,\lambda}$, except possibly near its endpoints.
\end{enumerate}
\label{def:spurious_bridge}
\end{definition}

\begin{assumption}[Non-injective forward map yields multiple explanations]
There exist $y_0 \in \mathcal Y$ and two nonempty compact sets $A_1,A_2 \subset \mathcal X$ such that
\begin{enumerate}
    \item $A_1 \cap A_2 = \varnothing$,
    \item $\mathrm{dist}(A_1,A_2) > 0$,
    \item $h(x)=y_0$ for all $x \in A_1 \cup A_2$,
    \item $p_X(x)$ is strictly positive on neighborhoods of both $A_1$ and $A_2$.
\end{enumerate}
This assumption encodes posterior multimodality induced by the inverse problem. It formalizes the case where a single observation admits multiple, well-separated explanations in $\mathcal X$, which is the regime of interest in this work.
\label{ass:multiple_explanations}
\end{assumption}

\begin{assumption}[Response gap away from the true branches]
There exist disjoint open neighborhoods $U_1 \supset A_1$ and $U_2 \supset A_2$, and a constant $\eta>0$, such that
\begin{equation}
\|h(x)-y_0\| \ge \eta,
\qquad
\forall x \in \mathcal X \setminus (U_1 \cup U_2).
\label{eq:response_gap}
\end{equation}
This assumption ensures that the ambiguity is localized near the true explanatory branches: outside these neighborhoods, the forward response remains uniformly separated from $y_0$. Under small observational noise, this yields a posterior whose high-density region is disconnected rather than merely diffuse.
\label{ass:response_gap}
\end{assumption}

\begin{assumption}[Continuity of the conditional fiber map]
For the fixed observation $y_0$, the map
$$
G_{y_0} : \mathbb R^{d_z} \to \mathbb R^{d_x}
$$
is continuous.
This is the natural regime for direct conditional generators parameterized by neural networks, flow matching, or ODE-based transports, where $x=G(z,y_0)$ depends continuously on the latent variable.
\label{ass:continuity}
\end{assumption}

\begin{assumption}[Connected latent high-density region]
For some $\lambda_Z>0$, the latent superlevel set $L^Z_{\lambda_Z}$ is connected.
This assumption is standard for unimodal base distributions, such as a Gaussian, and captures the common design in which posterior variability is generated from a single connected latent source.
\label{ass:connected_latent}
\end{assumption}

Assumptions~\ref{ass:multiple_explanations}--\ref{ass:response_gap} describe the geometry induced by the inverse problem, while Assumptions~\ref{ass:continuity}--\ref{ass:connected_latent} describe the structural constraints of a direct conditional parameterization. Under these assumptions, the posterior high-density set is disconnected, while the image of a connected latent high-density region under a continuous direct conditional map remains connected. This is what forces support distortion in direct conditional flows.

\subsubsection{Multimodal posterior geometry}

\begin{lemma}[Disconnected posterior branches under a non-injective inverse problem]
Under Assumptions~\ref{ass:multiple_explanations}--\ref{ass:response_gap}, there exists $\sigma_0>0$ such that for all $0<\sigma\le \sigma_0$, there exists $\lambda>0$ for which the posterior superlevel set $L_{y_0,\lambda}$ has at least two distinct connected components, one contained in a neighborhood of $A_1$ and the other contained in a neighborhood of $A_2$.
\label{lem:posterior_multibranch}
\end{lemma}

\begin{proof}
By Assumption~\ref{ass:multiple_explanations}, we have $h(x)=y_0$ for every $x\in A_1\cup A_2$. Therefore, for such $x$, the likelihood term $p_\varepsilon(y_0-h(x))$ is maximized. Since $p_X$ is strictly positive on neighborhoods of $A_1$ and $A_2$, the posterior density
$$
p(x\mid y_0)\propto p_X(x)\,p_\varepsilon(y_0-h(x))
$$
is strictly positive and comparatively large on neighborhoods of both sets.

On the other hand, by Assumption~\ref{ass:response_gap}, for every $x\in \mathcal X\setminus(U_1\cup U_2)$ we have $\|h(x)-y_0\|\ge \eta$. In the Gaussian-noise case, this yields
\begin{equation}
p(x\mid y_0)
\le
C \exp\!\left(-\frac{\eta^2}{2\sigma^2}\right)
\qquad
\forall x \in \mathcal X\setminus(U_1\cup U_2),
\label{eq:posterior_outside_gap}
\end{equation}
for some constant $C>0$ depending on an upper bound of the prior density. As $\sigma\to 0$, the right-hand side tends to zero.

Hence, for sufficiently small $\sigma$, one can choose $\lambda>0$ such that
\begin{enumerate}
    \item $p(x\mid y_0)\ge \lambda$ on neighborhoods of $A_1$ and $A_2$,
    \item $p(x\mid y_0)<\lambda$ on $\mathcal X\setminus(U_1\cup U_2)$.
\end{enumerate}
Therefore, $L_{y_0,\lambda}$ contains at least two disconnected high-density branches, one near $A_1$ and one near $A_2$.
\end{proof}

\subsubsection{Topological obstruction for direct conditional transport}

\begin{theorem}[Connected latent fibers cannot match a disconnected posterior component structure]
Under Assumptions~\ref{ass:multiple_explanations}--\ref{ass:connected_latent}, let $\lambda>0$ be chosen as in Lemma~\ref{lem:posterior_multibranch}. Then the image
$$
\mathcal G_{y_0,\lambda_Z} = G_{y_0}(L^Z_{\lambda_Z})
$$
is connected, and therefore cannot match the connected-component structure of the true posterior superlevel set $L_{y_0,\lambda}$ whenever $L_{y_0,\lambda}$ has at least two connected components.
\label{thm:topological_obstruction}
\end{theorem}

\begin{proof}
By Assumption~\ref{ass:connected_latent}, the latent high-density set $L^Z_{\lambda_Z}$ is connected. By Assumption~\ref{ass:continuity}, the fiber map $G_{y_0}$ is continuous. Since the continuous image of a connected set is connected,
$$
\mathcal G_{y_0,\lambda_Z} = G_{y_0}(L^Z_{\lambda_Z})
$$
must be connected.

By Lemma~\ref{lem:posterior_multibranch}, the true posterior superlevel set $L_{y_0,\lambda}$ has at least two connected components. Hence $\mathcal G_{y_0,\lambda_Z}$ cannot reproduce the connected-component structure of $L_{y_0,\lambda}$.
\end{proof}

\begin{corollary}[Three unavoidable failure modes of direct conditional transport]
Under the assumptions of Theorem~\ref{thm:topological_obstruction}, if the direct conditional model attempts to represent multiple true posterior branches simultaneously, then at least one of the following must occur:
\begin{enumerate}
    \item \textbf{mode dropping}: one or more true posterior branches are not represented;
    \item \textbf{spurious bridge}: the model creates an artificial connected path between distinct true branches through a low-density region;
    \item \textbf{off-support spurious mass}: the model assigns nontrivial mass to regions outside the true posterior high-density set.
\end{enumerate}
\label{cor:three_failures}
\end{corollary}

\begin{proof}
By Theorem~\ref{thm:topological_obstruction}, the image $\mathcal G_{y_0,\lambda_Z}$ is connected, whereas the true posterior superlevel set $L_{y_0,\lambda}$ contains at least two disconnected components. If the model intersects multiple true branches while remaining connected, it must connect them through some path. Such a path cannot lie entirely inside $L_{y_0,\lambda}$; otherwise $L_{y_0,\lambda}$ itself would be connected, contradicting Lemma~\ref{lem:posterior_multibranch}. Therefore, the connecting path must pass through regions of low true posterior density, producing either a spurious bridge or off-support spurious mass. If the model avoids such a connection, then it necessarily fails to represent at least one branch, yielding mode dropping.
\end{proof}

\subsubsection{How joint twist transport changes the parameterization pressure}

The obstruction above is specific to a direct conditional parameterization, where
for fixed $y_0$ the conditional variation is represented through an $x$-valued
fiber map
\[
    G_{y_0}: z_x \mapsto x .
\]
A joint/twist parameterization instead learns an augmented transport between
coupled variables,
\[
    T(z_x,y) = (x,z_y),
\]
rather than an $x$-valued conditional decoder alone. This does not eliminate the
topological difficulty after projecting onto the $x$-coordinates. Instead, it
changes the parameterization by learning a coupled terminal state in which each
generated target is paired with a likelihood-side coordinate. The intended effect
is to reduce the empirical pressure to encode all observation compatibility and
posterior variation solely through the $x$-coordinates.

The auxiliary variable $z_y$ does not add information about $x$ beyond the observation $y$. Rather, it provides the terminal observation-branch coordinate in the joint output space. For fixed $y_0$, consider the restricted map
$$
\widetilde G_{y_0}: z_x \mapsto (x,z_y).
$$
Its local linearization has Jacobian
$$
D\widetilde G_{y_0}(z_x)
=
\begin{pmatrix}
J_x(z_x)\\
J_{z_y}(z_x)
\end{pmatrix}.
$$
Hence, the local map contains both the sensitivity of the generated target and the sensitivity of the terminal observation branch. This should not be read as moving posterior uncertainty away from $x$; rather, it shows that the twist map learns a coupled augmented representation in which generated targets are paired with an observation-side residual coordinate.

This is natural in inverse problems where $(x,y)$ are paired through a forward relation such as
$$
y = h(x) + \varepsilon.
$$
In such settings, $x$ and $y$ share structured information through the joint data distribution. A joint twist transport therefore exploits this paired geometry directly, rather than forcing all posterior structure for fixed $y_0$ to be realized through a single connected latent fiber in $x$-space.

\begin{assumption}[Joint map regularity]
There exists a continuously differentiable map
$$
T:\mathbb R^{d_{z_x}}\times \mathcal Y
\to
\mathcal X \times \mathbb R^{d_{z_y}},
\qquad
T(z_x,y)=(x,z_y),
$$
that is locally invertible on the region of interest.
\label{ass:joint_invertible}
\end{assumption}

\begin{assumption}[Residual conditional variability is locally represented in augmented coordinates]
For a fixed observation $y_0$, let
$$
\widetilde G_{y_0}: z_x \mapsto (x,z_y)
$$
denote the restricted latent-to-output map, and assume that $\widetilde G_{y_0}$ is locally differentiable. We assume that there exists a nonempty subset
$$
\mathcal S_{y_0} \subset \mathcal Z_x
$$
such that for each $z_x \in \mathcal S_{y_0}$, there exists a direction $\delta z_x \neq 0$ satisfying
$$
\|J_x(z_x)\,\delta z_x\| \le \varepsilon_x,
\qquad
\|J_{z_y}(z_x)\,\delta z_x\| \ge \varepsilon_y,
$$
for some small $\varepsilon_x \ge 0$ and some $\varepsilon_y > 0$. Therefore, along certain local latent directions, the terminal observation branch changes nontrivially even when the local change in the generated target is small.
\label{ass:residual_variability}
\end{assumption}

Assumption~\ref{ass:joint_invertible} is natural for flow-based or ODE-based transports: if $T$ is realized as the endpoint map of a sufficiently regular flow, then local invertibility follows from standard ODE theory together with the inverse function theorem.

Assumption~\ref{ass:residual_variability} is motivated by the residual view of inverse problems. Under a Gaussian observation model, the normalized residual $(y-F(x))/\sigma_{\mathrm{obs}}$ has Gaussian scale, so a posterior sample can be associated with both a prior/reference coordinate and a likelihood-residual coordinate. The role of $z_y$ is not to introduce new information about $x$ and not to absorb posterior uncertainty away from $x$. Instead, $z_y$ provides the terminal observation-branch coordinate of the augmented transport. Assumption~\ref{ass:residual_variability} states that this branch has nontrivial local sensitivity, which is a structurally plausible condition for a coupled transport between $(z_x,y)$ and $(x,z_y)$.

\begin{remark}
    [Joint transport reduces the fiber-wise parameterization pressure]
Under Assumptions~\ref{ass:joint_invertible}--\ref{ass:residual_variability}, the joint twist transport provides a less
restrictive augmented parameterization than an $x$-valued conditional decoder.
This does not remove the topological difficulty in the projected $x$-space:
after projection, the generated $x$-samples are still evaluated through their
conditional distribution under the fixed observation. However, the augmented
state allows the learned transport to couple generated targets with a
likelihood-side coordinate, which can reduce the empirical tendency to represent
all conditional variation through the $x$-coordinates alone.
\label{prop:joint_alleviates}
\end{remark}

\begin{proof}[Proof sketch]
For a direct conditional model, fixing $y_0$ yields a single map
$$
G_{y_0}: z_x \mapsto x,
$$
so the conditional geometry in $x$-space must be realized entirely through the image of one latent fiber. This is precisely the setting of Theorem~\ref{thm:topological_obstruction}: if the latent high-density region is connected and the map is continuous, then its image remains connected, which creates difficulty when the true posterior high-density set is disconnected.

By contrast, a joint twist map
$$
T(z_x,y)=(x,z_y)
$$
acts between augmented spaces and preserves a coupled source--target structure. For fixed $y_0$, the relevant restricted map is
$$
\widetilde G_{y_0}: z_x \mapsto (x,z_y).
$$
Thus, the learned map is a full augmented transport rather than an $x$-valued decoder. In a local linearization,
$$
D\widetilde G_{y_0}(z_x)
=
\begin{pmatrix}
J_x(z_x)\\
J_{z_y}(z_x)
\end{pmatrix},
$$
the terminal observation branch can vary together with the generated target. Therefore, the model is not reduced to a single $x$-valued decoder in the same restrictive sense as a direct conditional map. While this does not remove the difficulty of representing a multi-branch posterior, it weakens the fiber-wise pressure that in the direct conditional setting tends to create spurious bridges or to discard branches.
\end{proof}

\begin{remark}[Compatibility with flow matching]
\label{rem:flow_matching_joint}
The joint twist parameterization is naturally compatible with flow matching \cite{lipman2022flow}. Let
$$
s_0=(z_x,y),
\qquad
s_1=(x,z_y).
$$
Consider a continuous probability path
$$
\{\pi_t^{\mathrm{joint}}(\cdot \mid s_0,s_1)\}_{t\in[0,1]}
$$
on $\mathbb R^{d_x+d_{z_y}}$ connecting $s_0$ and $s_1$, and let
$$
s_t \sim \pi_t^{\mathrm{joint}}(\cdot \mid s_0,s_1),
\qquad
u_t^{\mathrm{joint}}(s_t \mid s_0,s_1)\in\mathbb R^{d_x+d_{z_y}}
$$
denote the corresponding target velocity field along this path. Flow matching then trains a neural velocity field to regress
$$
u_t^{\mathrm{joint}}(s_t \mid s_0,s_1)
$$
under the induced path distribution. Thus, provided the joint data distribution admits a sufficiently regular transport representation, a continuous flow-matching model can in principle learn the corresponding joint map.
\end{remark}

\subsubsection{A least-squares residual interpretation of twist flow}
\label{app:twist_least_squares}

The discussion above explains why the auxiliary coordinate $z_y$ is useful for defining a full augmented transport. We now give a complementary optimization interpretation: the learned twist flow can be viewed as inducing a family of augmented least-squares consistency problems. This interpretation is not meant to claim that each flow-matching step is identical to a classical least-squares gradient step. Rather, it states that once the joint transport has been learned, its endpoint can be characterized as the minimizer of a learned augmented residual.

\paragraph{1. Flow matching learns a joint transport.}
Let
\begin{equation}
    s_0=(z_x,y),
    \qquad
    s_1=(x,z_y).
\end{equation}
A flow-matching model learns a time-dependent velocity field $v_\theta(s,t)$ whose ODE
\begin{equation}
    \frac{d s_t}{dt}=v_\theta(s_t,t),
    \qquad t\in[0,1],
\end{equation}
defines an endpoint map
\begin{equation}
    T_\theta:s_0\mapsto s_1^\theta .
\end{equation}
In the twist parameterization, this endpoint map has the form
\begin{equation}
    T_\theta:(z_x,y)\mapsto (x,z_y).
\end{equation}
Thus, flow matching is used to learn a transport between the augmented source variables $(z_x,y)$ and the augmented target variables $(x,z_y)$, rather than a one-way conditional decoder from $(z_x,y)$ to $x$ alone.

\paragraph{2. The same map induces inverse and forward solvers.}
For a fixed observation $y_0$, the forward integration of the learned ODE gives an inverse sampler
\begin{equation}
    (\hat x,\hat z_y)=T_\theta(z_x,y_0),
    \qquad
    \hat x_\theta(y_0,z_x)=\pi_xT_\theta(z_x,y_0).
\end{equation}
If the learned ODE is sufficiently regular, its endpoint map is locally invertible, and the reverse map satisfies
\begin{equation}
    T_\theta^{-1}:(x,z_y)\mapsto(z_x,y).
\end{equation}
Therefore, the same learned transport also induces a forward predictor
\begin{equation}
    \hat y_\theta(x,z_y)
    :=
    \pi_yT_\theta^{-1}(x,z_y).
\end{equation}
This is the sense in which the augmented transport also defines a reverse direction, although the main use in this work is fixed-observation posterior sampling.

\paragraph{3. The role of $z_y$ in a learned observation residual.}
The auxiliary coordinate $z_y$ should not be interpreted as additional information about $x$ or as a replacement for posterior uncertainty in $x$. Its role is to complete the augmented variable $(x,z_y)$ and to provide the terminal coordinate paired with the observation branch. Through the inverse map, $z_y$ participates in the learned observation prediction
\begin{equation}
    \hat y_\theta(x,z_y)=\pi_yT_\theta^{-1}(x,z_y),
\end{equation}
and therefore defines the learned observation residual
\begin{equation}
    r_y^\theta(x,z_y;y_0)
    :=
    \hat y_\theta(x,z_y)-y_0
    =
    \pi_yT_\theta^{-1}(x,z_y)-y_0.
\end{equation}
For deterministic observation operators, such as masked Fourier measurements $y=F(x)=P_\Omega\mathcal F x$, this residual should reduce to the physical consistency condition $F(x)\approx y_0$ in the ideal case. In such settings, $z_y$ does not represent physical measurement uncertainty; the missing-frequency or null-space variability of $x\mid y_0$ is instead indexed by $z_x$. The role of $z_y$ is primarily to make the joint transport full-dimensional and to define the observation-side consistency residual.

\paragraph{4. Endpoint samples minimize an induced least-squares energy.}
For fixed $(z_x,y_0)$, define the learned augmented consistency energy
\begin{equation}
\begin{aligned}
\mathcal E_\theta(x,z_y;z_x,y_0)
={}&
\frac{1}{2\sigma_y^2}
\left\|
\pi_y T_\theta^{-1}(x,z_y)-y_0
\right\|^2 \\
&\quad+
\frac{1}{2\sigma_z^2}
\left\|
\pi_{z_x} T_\theta^{-1}(x,z_y)-z_x
\right\|^2 .
\end{aligned}
\label{eq:twist_learned_ls_energy}
\end{equation}
The first term is a learned observation-consistency residual, while the second term enforces consistency with the latent coordinate that indexes posterior variability. Since both terms are nonnegative, and since
\begin{equation}
    T_\theta^{-1}\!\big(T_\theta(z_x,y_0)\big)=(z_x,y_0),
\end{equation}
we have
\begin{equation}
    \mathcal E_\theta\!\big(T_\theta(z_x,y_0);z_x,y_0\big)=0.
\end{equation}
Therefore, under local invertibility of $T_\theta$,
\begin{equation}
    T_\theta(z_x,y_0)
    =
    \arg\min_{x,z_y}
    \mathcal E_\theta(x,z_y;z_x,y_0).
\label{eq:twist_endpoint_ls_minimizer}
\end{equation}
This gives the desired least-squares interpretation: each posterior sample is not obtained by explicitly solving a hand-crafted least-squares problem, but it is the endpoint of a learned transport whose inverse defines an augmented least-squares residual minimized at that endpoint.

\paragraph{5. Relation to least-squares updates.}
Equation~\eqref{eq:twist_endpoint_ls_minimizer} characterizes the endpoint of the learned flow. It does not imply that the velocity field is exactly the Euclidean gradient of the energy in \eqref{eq:twist_learned_ls_energy}. A safer interpretation is that the learned velocity provides an implicit, preconditioned update direction for this consistency geometry. If $w=(x,z_y)$, one may view the reverse-side dynamics locally as
\begin{equation}
    \frac{dw_t}{dt}
    \approx
    -M_t(w_t)\nabla_w\mathcal E_{\theta,t}(w_t;z_x,y_0),
\end{equation}
where $M_t(w_t)$ is a learned metric or preconditioner and $\mathcal E_{\theta,t}$ denotes a time-dependent implicit consistency energy. Thus, the twist velocity is not necessarily an ordinary least-squares gradient, but it can be interpreted as a learned preconditioned least-squares update direction. In this sense, the connection between $z_y$ and $y$ inside the twist map provides a learned observation-consistency constraint, while $z_x$ selects different posterior-consistent solutions.

\subsection{Toy inverse-problem details}
\label{app:toy_details}
\label{app:analytical_prior_mcmc}

This appendix provides the data-generation details and reference posterior
construction for the low-dimensional toy experiments in
Section~\ref{sec:toy_validation}. The purpose of these experiments is to isolate
posterior support preservation in a controlled setting where the prior and
likelihood are analytically specified.

\subsubsection{Data generation and analytical priors}
\label{app:toy_data_generation}

We use three two-dimensional prior distributions over
$x=(x_1,x_2)\in\mathbb{R}^2$: \texttt{two\_moons}, \texttt{ring}, and
\texttt{triangles}. These priors are chosen because they induce posterior
supports with different structures: curved disconnected branches, closed
annular support, and separated cluster-like modes.

\paragraph{Triangular mixture.}
For the \texttt{triangles} setting, the prior is a three-component Gaussian
mixture,
\begin{equation}
p_X(x)
=
\sum_{k=1}^{3}\pi_k\,\mathcal N(x;\mu_k,\Sigma_k),
\label{eq:triangle_prior}
\end{equation}
with equal mixture weights and component means arranged in a triangular
configuration. This produces a disconnected multimodal prior with three
dominant modes.

\paragraph{Ring distribution.}
For the \texttt{ring} setting, the prior is defined in polar coordinates by a
radial Gaussian law together with a uniform angular variable:
\begin{equation}
p_X(x)
=
\frac{1}{2\pi}\cdot \frac{1}{r}\,
\mathcal N(r;\mu_r,\sigma_r^2),
\qquad
r=\|x\|.
\label{eq:ring_prior}
\end{equation}
This induces a continuous density concentrated around an annulus, providing a
prior with nontrivial global support structure.

\paragraph{Two-moons distribution.}
For the \texttt{two\_moons} setting, the prior is defined through a latent branch
variable $b\in\{0,1\}$ and a curve parameter $t\in[0,\pi]$:
\begin{equation}
p_X(x)
=
\frac{1}{2}\sum_{b\in\{0,1\}}
\int_{0}^{\pi}
\mathcal N\!\left(x;m_b(t),\sigma_x^2 I\right)\frac{dt}{\pi}.
\label{eq:twomoons_prior}
\end{equation}
Here $m_0(t)$ and $m_1(t)$ parameterize the two moon-shaped arcs. Thus the prior
is a continuous density concentrated near two curved branches rather than a
finite Gaussian mixture.

\paragraph{Observation model.}
Across all toy settings, observations are generated from the scalar nonlinear
forward model
\begin{equation}
    y = h(x) + \varepsilon,
    \qquad
    h(x)=c_1x_1^2+c_2x_2,
    \qquad
    \varepsilon \sim \mathcal N(0,\sigma^2).
\label{eq:toy_obs_model_appendix}
\end{equation}
The coefficients $(c_1,c_2)$ and noise level $\sigma$ are setting-specific, but
the same observation family is used across settings. This design keeps the
inverse problem comparable while changing the conditional posterior support
through the prior over $x$.

For each setting, we first sample $x$ from the prior and then sample $y$ from the
observation model. During evaluation, fixed observations $y_0$ are selected from
the empirical observation distribution. The learned models are then evaluated by
generating samples from $p_\theta(x\mid y_0)$ and comparing them to reference
posterior samples for the same $y_0$.

\subsubsection{Reference posterior sampling with MCMC}
\label{app:toy_mcmc}

For each fixed observation $y_0$, the posterior density is available up to a
normalizing constant:
\begin{equation}
p(x\mid y_0)
\propto
p_X(x)
\exp\!\left(
-\frac{
\left(y_0-c_1x_1^2-c_2x_2\right)^2
}{2\sigma^2}
\right).
\label{eq:toy_posterior_closed_form}
\end{equation}
Because both the prior $p_X(x)$ and the observation model are specified
analytically, $\log p(x\mid y_0)$ and its gradient can be evaluated directly in
the toy setting.

We use Markov chain Monte Carlo as a reference posterior sampler targeting
\eqref{eq:toy_posterior_closed_form}. In particular, we use the
Metropolis-adjusted Langevin algorithm (MALA) \cite{roberts1996exponential},
which proposes gradient-informed moves using the analytical posterior score and
then applies a Metropolis correction. Under standard regularity conditions, the
chain has invariant distribution $p(x\mid y_0)$, and therefore provides a
principled posterior reference for these low-dimensional examples.

The MCMC samples are used only for evaluation. They are not introduced as a
competing learned generative model. After burn-in and thinning, the retained
samples are used to compute the marginal Wasserstein distances, KS statistics,
mean error, covariance error, and posterior overlays reported in the main text.

\paragraph{Scope.}
This analytical-prior and MCMC-reference construction is specific to the
low-dimensional toy experiments. In the high-dimensional image and scientific
inverse problems, exact posterior access is unavailable, so evaluation instead
relies on indirect diagnostics such as observation consistency, sample diversity,
and uncertainty structure. The toy setting is therefore used as a controlled
validation of posterior support preservation.

\subsection{Seismic data simulation and RTM construction}
\label{app:seis_details}

This appendix provides the data-generation, RTM construction, and evaluation
details for the seismic inverse problem experiments in
Section~\ref{sec:seismic}. In these experiments, the target variable \(x\) is a
subsurface velocity model, while the conditioning observation \(y\) is an
image-domain reverse-time migration (RTM) image derived from simulated seismic
shot records. Thus, the generative model is conditioned on an imaging product,
rather than directly on the raw seismic data.

\subsubsection{Reverse-time migration}

For each velocity model \(x\), we first simulate seismic shot records using
acoustic wave-equation modeling and then apply RTM to obtain the image-domain
observation \(y\). Let \(q\)
denote the source wavelet. Following the \textit{JUDI.jl} operator notation, the acoustic
modeling operator can be written as
\[
    d_{\mathrm{obs}} = P_r F(x) P_s^\top q + \epsilon ,
\]
where \(d_{\mathrm{obs}}\) denotes the observed shot records, \(F(x)\) denotes
acoustic wave-equation propagation in the subsurface velocity \(x\), \(P_s^\top\) injects the
source wavelet into the computational domain, \(P_r\) samples the wavefield at
receiver locations, and \(\epsilon\) denotes acquisition noise.

The noisy shot records are migrated using RTM with a fixed smooth background
model \(x_0\). In operator form, the RTM image is computed by applying the
adjoint of the Born modeling operator around \(x_0\) to the recorded data,
\[
    y = J(x_0,q)^\top d_{\mathrm{obs}},
\]
where \(J(x_0,q)\) denotes the linearized Born modeling operator, or Jacobian,
defined by the smoothed background model \(x_0\) and source wavelet \(q\). For the linearized modeling, we instead apply the adjoint Born operator to the data residual,
\[
    y = J(x_0,q)^\top (d_{\mathrm{obs}} - d_0),
    \qquad
    d_0 = P_r F(x_0)P_s^\top q ,
\]
where \(d_0\) is the predicted data generated in the smooth background model.
The resulting RTM image \(y\) serves as the conditioning observation for the
generative model.

\subsubsection{Data simulation details}

We construct the seismic dataset using 800 two-dimensional slices from the three-dimensional Compass dataset for training \cite{jones2012building}. Each slice covers a physical domain of \(6.4\,\mathrm{km} \times 3.2\,\mathrm{km}\) and is discretized on a
\(512 \times 256\) grid. Each slice is treated as a subsurface velocity model
\(x\). The surface is assumed to be absorbing. We use a \(15\,\mathrm{Hz}\) Ricker wavelet as the source signature \(q\). The seismic acquisition setup consists of 512 receivers and 64 sources. To simulate acquisition noise, we add uncorrelated band-limited Gaussian noise with a signal-to-noise ratio of \(12\,\mathrm{dB}\) to the shot records. For the migration, we use the fixed background model \(x_0\) obtained by first averaging the 800 velocity slices and then applying a Gaussian filter with a 10-sample bandwidth for all the velocity models.

\subsection{Additional reproducibility details}
\label{app:reproducibility}

We provide additional implementation details to support reproducibility of the
main experimental results. All learning-based experiments compare the proposed
joint twist-flow model with a direct conditional-flow baseline using the same
training pairs, preprocessing, optimizer, architecture width, and evaluation
protocol whenever possible. 

\paragraph{Software.}
The neural models are implemented in PyTorch. ODE integration uses
\texttt{torchdiffeq}. Numerical metrics and post-processing use
\texttt{NumPy}, \texttt{SciPy}, \texttt{scikit-image}, \texttt{scikit-learn},
and \texttt{Matplotlib}. The software environment uses Python 3.10.18, PyTorch
2.5.1 with CUDA 12.4, torchvision 0.20.1, and \texttt{torchdiffeq} 0.2.5.
The seismic wave-equation simulations and RTM-image generation use Devito and
JUDI.jl. All random-number generators are seeded before data generation,
training, and posterior sampling.

\paragraph{Image model architecture.}
For the CelebA image-restoration experiments, both CondFlow and joint twist-flow
use the same U-Net backbone. The backbone uses
\[
\texttt{model\_channels}=128,\quad
\texttt{channel\_mult}=[1,2,3,4],\quad
\texttt{num\_blocks}=2,
\]
with positional time embeddings and embedding dimension 128. CondFlow takes
\((x_t,y)\) as input, with 6 input channels and 3 output channels. Joint
twist-flow takes \((x_t,y_t)\) as input, with 6 input channels and 6 output
channels, corresponding to the velocity of the augmented state \((x_t,y_t)\).
The same architecture width and depth are used for both methods.

\paragraph{Image training.}
All image models are trained with AdamW using learning rate \(10^{-4}\), batch
size 32, and 300 epochs. The learning rate follows a cosine annealing schedule
with \(T_{\max}=100\) and minimum learning rate \(10^{-6}\). The flow-matching
time variable is sampled as \(t\sim\mathcal{U}[0,1]\). For each training pair
\((x,y)\), we draw \(z_x\sim\mathcal{N}(0,I)\). For joint twist-flow, we also
draw \(z_y\sim\mathcal{N}(0,I)\), corresponding to \(\sigma_y=1\), and define
the linear interpolation
\[
    x_t=(1-t)z_x+t x,\qquad
    y_t=(1-t)y+t z_y.
\]
The target velocity is
\[
    u_t^{\rm twist} = (x-z_x,\; z_y-y).
\]
For CondFlow, we use
\[
    x_t=(1-t)z_x+t x,\qquad
    u_t^{\rm cond}=x-z_x,
\]
with the fixed observation \(y\) concatenated as a conditioning input. Both
models minimize the mean-squared error between the predicted and target
velocities.

\paragraph{ODE sampling.}
At test time, posterior samples are generated by integrating the learned ODE over
\(t\in[0,1]\). For the image experiments, we use Euler integration with 100
time steps for both CondFlow and joint twist-flow. For a fixed observation,
CondFlow starts from \(z_x\), while joint twist-flow starts from
\((z_x,y)\). The final \(x\)-component is used as the generated posterior sample.
Checkpoints are saved every 50 epochs.

\subsection{Image-restoration details}
\label{app:image_extra_details}

For the image-restoration experiments, we use CelebA images as clean targets.
Images are center-cropped to \(178\times178\), resized to \(64\times64\), and
normalized to \([-1,1]\). We construct a shared split from the sorted image list:
90\% of the images are used for training and the remaining 10\% for validation.
A shared manifest of 20 held-out validation images is saved and reused across
observation operators, so that different restoration tasks use the same clean
test images but task-specific observations.

We evaluate three observation operators. For low-frequency restoration, we apply
a centered Fourier low-pass mask with keep ratio 0.18. For blur-downsampling, we
apply Gaussian blur with kernel size 21 and standard deviation 10.0, downsample
to \(16\times16\), and bilinearly upsample back to \(64\times64\). For
blur-masking, we use a center mask of size \(24\times24\), fill the masked
region with 0, and pre-blur the image before masking using the same
Gaussian blur setting.

\subsection{Compute resources}
\label{app:compute}

The experiments were run on an institutional SLURM GPU cluster using NVIDIA
A100 GPUs.  A typical
training job requested one GPU, one CPU core, 40GB memory, and a wall-time limit
of 200 hours:
\[
\texttt{--partition=gpu},\quad
\texttt{--gres=gpu:1},\quad
\texttt{--cpus-per-task=1},\quad
\texttt{--mem=40G},\quad
\texttt{--time=200:00:00}.
\]
The wall-time value is the requested upper limit rather than the exact runtime
of every job. The reported image experiments train one CondFlow model and one
joint twist-flow model for each observation operator. The seismic experiments
use the same GPU training setup, with additional Devito/JUDI wave-equation
simulation and forward-modeling post-processing for the observation-space
consistency checks. Additional preliminary runs were used for hyperparameter
selection and figure generation but used the same model families and datasets as
the reported experiments.

\subsection{Broader impacts}
\label{app:broader_impacts}

This work develops posterior-sampling methods for inverse problems. A potential
positive impact is improved uncertainty characterization in scientific and
engineering applications, where a single deterministic reconstruction can hide
ambiguity caused by incomplete or degraded observations. In seismic imaging and
subsurface monitoring, posterior ensembles may help identify regions where the
available observations weakly constrain the model.

The main risk is that generated samples may be over-interpreted as calibrated
posterior truth. Although our experiments evaluate posterior spread and
observation-space consistency, they do not establish perfect calibration in all
high-dimensional settings. Therefore, the method should not be used as the sole
basis for high-stakes scientific, environmental, or engineering decisions without
additional calibration checks, domain validation, and application-specific
uncertainty diagnostics.

\subsection{Existing assets and licenses}
\label{app:assets}

We use existing datasets and open-source software packages. Table~\ref{tab:assets_licenses}
summarizes their role and license or terms of use. We do not redistribute CelebA
images or derived image datasets; users should obtain CelebA from its original
source and follow its terms of use.

\begin{table}[h]
\centering
\small
\caption{Existing assets used in this work.}
\label{tab:assets_licenses}
\begin{tabular}{p{0.22\linewidth}p{0.36\linewidth}p{0.34\linewidth}}
\toprule
Asset & Use in this paper & License or terms of use \\
\midrule
CelebA
& Image-restoration experiments.
& Non-commercial research use only; images are not redistributed. \\
Devito
& Acoustic wave-equation simulation.
& MIT License. \\
JUDI.jl
& Seismic modeling and RTM-related simulation workflow.
& MIT License. \\
PyTorch
& Neural-network training and inference.
& BSD-style open-source license. \\
\texttt{torchdiffeq}
& ODE integration for continuous-time sampling.
& MIT License. \\
NumPy, SciPy, scikit-learn, Matplotlib
& Numerical computation, metrics, and plotting.
& Open-source scientific Python packages under permissive licenses. \\
\bottomrule
\end{tabular}
\end{table}

\end{document}